\documentclass[nohyperref]{article}

\usepackage{amsmath}
\DeclareMathOperator*{\argmax}{arg\,max}
\DeclareMathOperator*{\argmin}{arg\,min}
\usepackage{microtype}
\usepackage{ amssymb }
\usepackage{graphicx}
\usepackage{subfigure}
\usepackage{booktabs} 
\usepackage{multirow}
\usepackage{graphbox}

\PassOptionsToPackage{hyphens}{url}\usepackage{hyperref}

\usepackage[accepted]{icml2022}

\usepackage{amsmath}
\usepackage{amssymb}
\usepackage{mathtools}
\usepackage{amsthm}
\DeclareMathOperator{\norm}{Norm}
\DeclareMathOperator{\ks}{KS}
\DeclareMathOperator{\similarity}{Sim}
\usepackage[capitalize,noabbrev]{cleveref}

\theoremstyle{plain}
\newtheorem{theorem}{Theorem}[section]
\newtheorem{proposition}[theorem]{Proposition}

\theoremstyle{definition}

\theoremstyle{remark}
\newtheorem{remark}[theorem]{Remark}
\newtheorem*{proposition*}{Proposition}

\usepackage[textsize=tiny]{todonotes}

\icmltitlerunning{Fishing for User Data in Large-Batch Federated Learning}

\begin{document}

\twocolumn[
\icmltitle{Fishing for User Data in Large-Batch Federated Learning\\ via Gradient Magnification}



\icmlsetsymbol{equal}{*}

\begin{icmlauthorlist}
\icmlauthor{Yuxin Wen}{equal,umd}
\icmlauthor{Jonas Geiping}{equal,umd}
\icmlauthor{Liam Fowl}{equal,umd}
\icmlauthor{Micah Goldblum}{nyu}
\icmlauthor{Tom Goldstein}{umd}
\end{icmlauthorlist}

\icmlaffiliation{umd}{University of Maryland}
\icmlaffiliation{nyu}{New York University}

\icmlcorrespondingauthor{Yuxin Wen}{ywen@umd.edu}
\icmlcorrespondingauthor{Jonas Geiping}{jgeiping@umd.edu}
\icmlcorrespondingauthor{Liam Fowl}{lfowl@umd.edu}

\icmlkeywords{Machine Learning, ICML}

\vskip 0.3in
]



\printAffiliationsAndNotice{\icmlEqualContribution} 

\begin{abstract}
Federated learning (FL) has rapidly risen in popularity due to its promise of privacy and efficiency. Previous works have exposed privacy vulnerabilities in the FL pipeline by recovering user data from gradient updates. However, existing attacks fail to address realistic settings because they either 1) require toy settings with very small batch sizes, or 2) require unrealistic and conspicuous architecture modifications.
We introduce a new strategy that dramatically elevates existing attacks to operate on batches of arbitrarily large size, and without architectural modifications. Our model-agnostic strategy only requires modifications to the model parameters sent to the user, which is a realistic threat model in many scenarios. 
We demonstrate the strategy in challenging large-scale settings, obtaining high-fidelity data extraction in both cross-device and cross-silo federated learning. 
Code is available at \url{https://github.com/JonasGeiping/breaching}.

%
%
%
\end{abstract}

\section{Introduction}
\label{intro}
Is it possible to train machine learning models on massive amounts of private data without compromising the data and feeding it to a central server? Federated learning is poised to be one step towards a solution to this question. 

\begin{figure}[h!]
    \centering
    \includegraphics[width=0.45\textwidth]{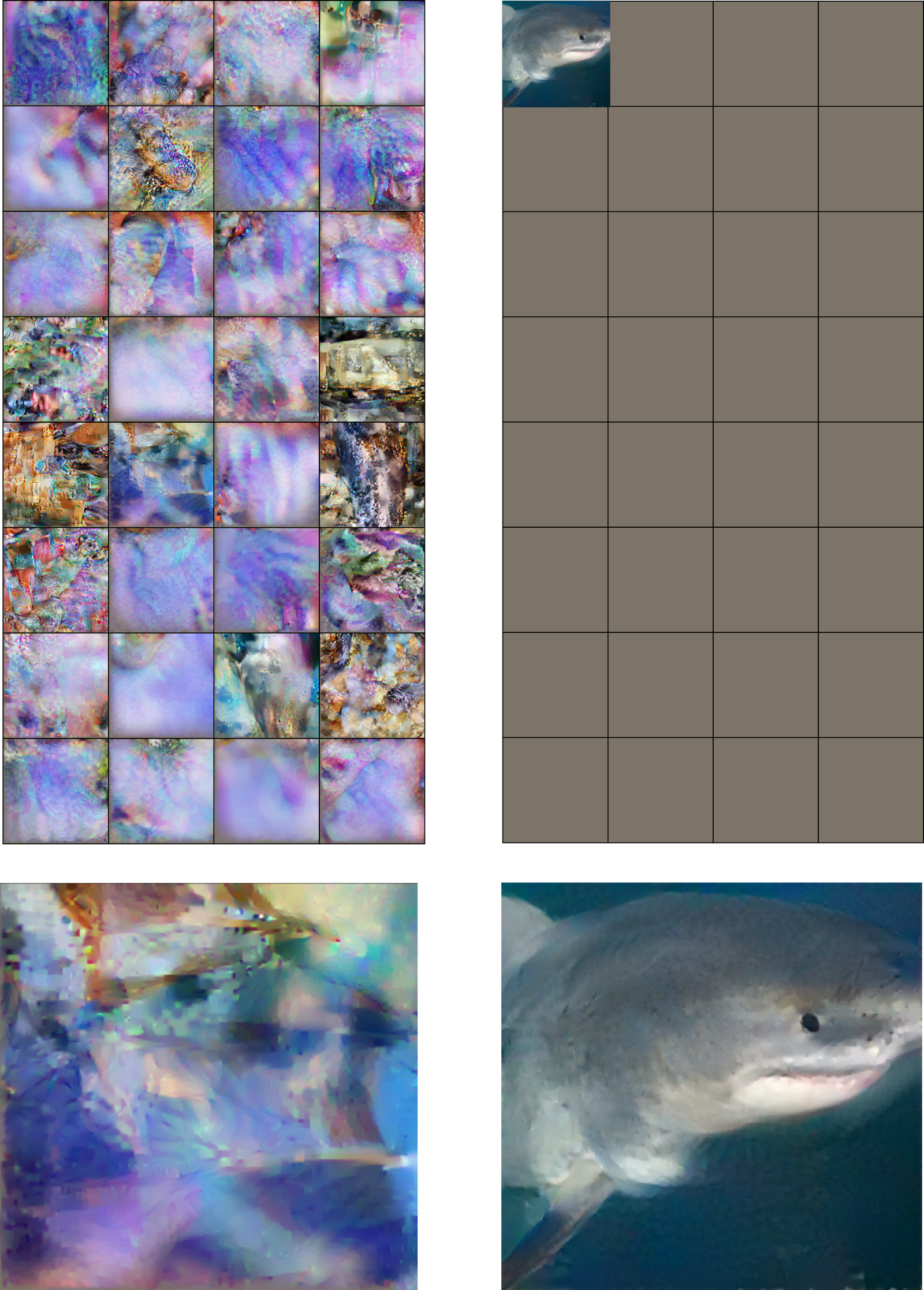}
    \caption{Previous attacks (left) for a batch size of 32 in the ``honest-but-curious'' threat model, vs. a fishing attack (right) under the ``malicious-parameters'' threat model. Malicious modifications of the server update lead to the complete compromise of a single example which is \textit{fished} out of an arbitrarily large batch of aggregated data.}
    \label{fig:teaser}
\end{figure}
In federated learning (FL) and other collaborative learning frameworks, machine learning models are trained in a decentralized manner. A central server sends only the current state of the model to all participating users, the users compute model updates based on their own private data, and only their update is returned to the server, optionally after being securely aggregated across multiple users \citep{kairouz_advances_2021,bonawitz_practical_2017}. This is often referred to as the principle of data minimization - the server never sees user data directly, and model updates sent by the user are considered to only contain the minimal information necessary to improve the model \citep{bonawitz_federated_2021}.
However, federated learning is only secure as long as the model updates sent by users cannot be inverted to retrieve their original contents. Such attacks breach privacy in federated learning and directly reveal user data. Previously, such attacks have been demonstrated mostly on either small batches of data, e.g. in \citet{wang_beyond_2018,zhu_deep_2019} under the threat model of an ``honest-but-curious" server or in settings with ``malicious model" threat models  with modified architectures \citep{fowl_robbing_2021,boenisch_when_2021}. For the threat model of an unmodified architecture, existing work seems to suggest that even a moderate batch size of 32 could be sufficient to deter an attack \citep{huang_evaluating_2021}.

In this work, we discuss a third category of threat model, ``malicious parameters". In this setting, the server cannot be trusted to send benign model updates to users. From a user perspective, this is a realistic scenario; the server update is outside of their control, computed on central machines by a company with potentially limited intent to uphold privacy. Such a company could publicly support federated learning, for example to comply with regulations such as GDPR \citep{truong_privacy_2021}, but still obtain private user information through attacks such as the one we discuss hereafter. As such, we argue that a user should regard server updates as untrusted.

We introduce a family of attacks that are able to fish for single data points from a large aggregate, magnifying the gradient contribution of a target data point and reducing the gradient contribution of other data through maliciously chosen parameter vectors. We discuss this attack for both common federated learning scenarios of cross-device and cross-silo learning. We then verify in a range of experiments for the task of image classification that this attack allows us to leverage existing optimization-based \citep{geiping_inverting_2020-1} and analytic attacks \citep{lu_april_2021}, which currently only work well for inverting an updated calculated on a single or few data points.

\section{Background and Related Work}
\label{previous_work}
We are interested in two main settings of federated learning \citep{kairouz_advances_2021}. In the \textit{cross-device} scenario, a shared model is trained on a massive number of users, e.g. mobile devices, which each hold a small amount of private data, but are highly unreliable and on average are not expected to participate more than once in training given their number \citep{bonawitz_towards_2019}. In particular, these users cannot be queried or addressed separately, and are selected at random. To increase privacy, contributions of groups of users can be aggregated securely \citep{mcmahan_communication-efficient_2017}. In contrast, in the \text{cross-silo} scenario, only a small group of users (e.g. $<100$) exists, who are reliable and identifiable and can be queried in every round of federated learning. An often-quoted example here is that consortia of hospitals train models without permission to share sensitive patient data directly \citep{jochems_distributed_2016}. These users each own large amounts of private data, which they aggregate themselves and protect in this way without having to rely on aggregation with other users.

Gradient inversion attacks against federated learning systems have been demonstrated for some time now and in several scenarios. Most work so far though has focused on the threat model of an        ``honest-but-curious" server, in which the server participates in the federated learning protocol as designed, but records and processes updates sent by users. \citet{phong_privacy-preserving_2017-1,geiping_inverting_2020-1} discuss how analytical methods can recover the input data to a linear layer anywhere in the network. Analytical attacks can also directly recover input data for several vision transformer architectures \citep{lu_april_2021}. For CNNs, analytical attacks can be extended recursively \citep{zhu_r-gap_2021} to resolve input data from convolutional layers. However, all of these attacks can only invert the gradients of single data points, recovering a mixture of all aggregated data, when the gradient is aggregated over many examples as in realistic settings as discussed before.

Optimization-based attacks as introduced in \citet{wang_beyond_2018} fare better, but can also only invert gradients for small batch sizes, such as 8 in \citep{zhu_deep_2019}. Other attacks can recover a few samples from larger batches when attacking larger models \citet{geiping_inverting_2020-1,yin_see_2021}, but overall, these methods also cannot invert gradients of large batches. This limits the attack vector of these attacks and allows for an effective defense through (secure) aggregation of user gradients. 

Recently, other threat models have also come into focus, such as threat models focused on malicious parameter vectors. As discussed, this is a realistic threat for users who should view any updates sent by the server as unverified parameters that the users deploy directly on their private data. Attacks in these threat models can recover data from arbitrarily large batch sizes as shown in \citet{fowl_robbing_2021} and as such overcome secure aggregation (more formally, the notion that aggregation itself is a defense is overcome, reducing secure aggregation in these scenarios to a multi-party shuffle - the outcome of secure aggregation is attacked, but not the protocol itself) . However, this previous work often requires unrealistic model modifications. Attacks in \citet{fowl_robbing_2021} and \citet{boenisch_when_2021} can only recover data from linear layers and thus require the existence of large linear layers - even one of which could be larger than standard vision models - early in a network, to avoid downsampling and striding operations that degrade information. The attack in \citet{pasquini_eluding_2021} requires the ability of the server to send separate malicious parameters to individual users who each own only few data points - a threat that can in turn be overcome quickly if aggregation protocols are used in reverse to average server updates before processing them on the user side. The attack in \citet{lam_gradient_2021} can break secure aggregation without these previous assumptions, but requires some knowledge of side-channel information that is not strictly necessary for the FL protocol. 

The attacks in \citet{fowl_robbing_2021,boenisch_when_2021} and \citet{pasquini_eluding_2021} can collectively be understood as attacks based on gradient sparsity. Although the aggregation protocol nominally runs as it was designed, all but one of the data points \citep{boenisch_when_2021} or users \citep{pasquini_eluding_2021} return a zero gradient for some parameters in the model, such that averaging still returns these entries directly. In \citet{fowl_robbing_2021}, individual gradients are non-zero, but parts of the gradient obey a cumulative sum structure from which a sparse single update can be recovered. These attacks dense aggregation into a sparse aggregation and extract separate gradients per data point, which can then be inverted by the methods described above. 

\looseness -1 In this work, we show that this methodology can be extended much more generally that previously believed. We describe attack strategies that can fish for single gradients from arbitrarily large aggregated batches and, crucially, apply to arbitrary models, without requiring the model modifications of \citet{fowl_robbing_2021}, and without requiring separate updates sent to all users \citet{pasquini_eluding_2021}, or knowledge of side-channel information as in \citet{lam_gradient_2021}. Even attacks such as \citet{lu_april_2021} that previously were not applicable to more than one data point, can now threaten arbitrary large aggregations.

\section{Method}
\label{sec:method}
In this section, we describe the two straightforward mechanisms of our attack strategy. These two aspects allow an attacker to effectively reduce a aggregated gradient to an update calculated on a \emph{single} data point. These combined strategies allow an attacker to breach privacy in both cross-device and cross-silo settings, even with large batches. In this work, we focus on tasks that can be posed as multi-class classification problems.

\paragraph{Threat Model:}
In this work we assume that users communicate with the server through established federated learning protocols (we focus mainly on \textit{fedSGD}) which are orchestrated in secure compute environments \citep{frey_introducing_2021} leaving the protocol itself as the only avenue of attack for the server. The server is intent on recovering private user information and is allowed to send modified updates to all users. In the cross-device setting, these updates can be sent to multiple groups of users, but each user is not expected to be queried more than once. In the cross-silo setting, modified updates can be sent successively as queries over multiple rounds of training, as all users participate in all rounds in this setting.

\subsection{Class fishing strategy}
\label{sec:cls_strategy}

We first introduce a class fishing strategy wherein the gradients for a single target class dominate the aggregated model update. The intuition behind this strategy is for the malicious server to artificially decrease the confidence of the network's predictions for the target class, thus significantly increasing the contribution of the gradient information calculated from data belonging to the selected class. In fact, the server can make the relative size of the selected class' gradient arbitrarily large in this manner. 

More formally, given a target class $c$, and fully connected classification layer $W \in \mathbb{R}^{n \times m}$ with bias $b \in \mathbb{R}^{n}$, the class fishing strategy modifies the parameters
\[
    W_{i, j}= 
\begin{cases}
    W_{i, j},            & \text{if } i = c\\
    0,              & \text{otherwise}
\end{cases}
\]
\[
    b_{i}= 
\begin{cases}
    b_i,            & \text{if } i = c\\
    \alpha,              & \text{otherwise}
\end{cases},
\]
where $\alpha$ is a server-side hyperparameter which controls the relative magnitude of the selected class's gradient signal. Let $x_i$ denote an image from target class $c$. If $\alpha$ is chosen so that $\langle W_j, x_i \rangle + b_j = \alpha,\ \forall j \neq i$, and $\langle W_i, x_i \rangle + b_i \ll \alpha$, then we have that the ratio $\frac{p_i}{p_o} \ll 1,\ \forall o \neq i$ where $p_i$ is the softmax output for $x_i$ corresponding to the target class $c$, and $p_o$ denotes the softmax output of any other image in the batch corresponding to its ground-truth label ($\neq c$). That is to say, if the server chooses a large enough bias $\alpha$ for the non-target classes, the ratio of the ground-truth class' softmax entry for an image coming from the target class is much lower than the ground-truth class' softmax entry for any other class. This leads to a predictable model update that is dominated by the contributions of one class. We quantify this in \cref{prop1} which decomposes a model update that is aggregated over images from many different classes as being almost exclusively from the target class. 

\begin{proposition}
\label{prop1}
Let $\phi$ denote the parameters of the feature extractor for the central model, and let $W, b$ be the weight and bias vectors, respectively, of the last classification layer. Furthermore, let $\{x_i\}_{i=1}^N$ be the collection of user data on which model updates can be calculated. Let $\{x_{i}\}_{i=1}^{m \leq N}$ be any batch data containing $k$ data points $\{x'_i\}^{k\leq m}_{i=1}$ from the target class. Then, for cross-entropy loss $\mathcal{L}$, if the gradient of $\mathcal{L}$ w.r.t. each data point in the batch is bounded in $l_2$ norm, then $\forall \varepsilon > 0$, $\exists \alpha$ so that a malicious server employing our class strategy on $W, b$ recovers a feature extractor update $g$ so that 
\[
\lVert g - \frac{1}{m}\sum_{i=1}^{k}\nabla_\phi \mathcal{L}(x'_i, c; \phi, W,b) \rVert_2 < \varepsilon.
\]
\end{proposition}
Interestingly, the choice of parameters defined above already improves the performance of optimization-based attacks such as \citet{geiping_inverting_2020-1}, even when operating on a single image. 

\begin{figure}[h!]
    \centering
    \includegraphics[width=0.48\textwidth]{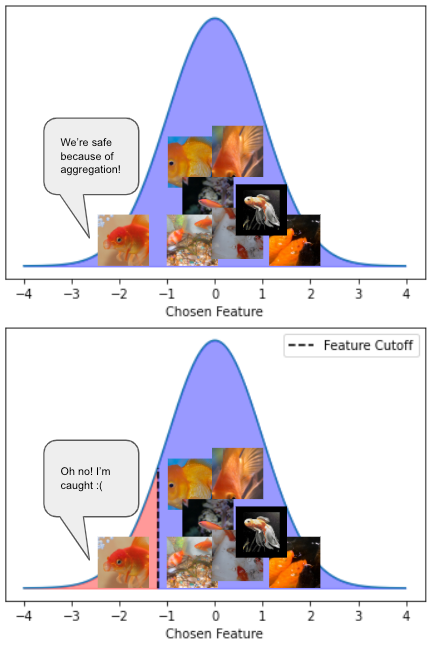}
      \caption{A high-level schematic of our one-shot binary attack. Within a class, ``Goldfish", features can be suppressed to reveal a single user image.}
      \label{fig:feature_teaser}
\end{figure}

\subsection{Feature fishing strategy}
\label{sec:feature_strategy}
Our previous class fishing strategy has a good chance of succeeding when the number of classes in the underlying label-space is large compared to the batch size on which updates are calculated. This is the case for the setting of previous attacks which have focused on inverting model updates for ImageNet \citep{yin_see_2021}. However, if there are intra-batch collisions in label space, or other users also have data from the target class, another strategy is needed as solely applying the class fishing strategy results in a ``mixed" update on all images from the target class. To that end, we introduce a feature fishing strategy to selectively retrieve gradients for images with the same class label in the same batch. Similar to the class fishing attack, this strategy allows a malicious server to magnify the gradients for a target image by appropriately shifting the decision boundary to decrease the confidence of the network's predictions for said target images. Like the class fishing strategy, the malicious server only needs to modify the parameters of the classification layer. 

While this strategy can succeed by adjusting the decision boundary in many ways, we opt to focus on selecting a \emph{single} feature along which to move the boundary. In this case, all other features are ignored by the network. Formally, given a target class $c$, a pre-selected feature entry index $k$, and fully connected classification layer $W \in \mathbb{R}^{n \times m}$ with bias $b \in \mathbb{R}^{n}$, the feature fishing strategy modifies the parameters
\[
    W_{i, j}= 
\begin{cases}
    \beta,            & \text{if } i = c \text{ and } j = k\\
    0,              & \text{otherwise}
\end{cases}
\]
\[
    b_{i}= 
\begin{cases}
    -\beta * \theta,            & \text{if } i = c\\
    0,              & \text{otherwise}
\end{cases},
\]
where $\theta$ and $\beta$ are two server-side hyperparameters for choosing the target images and controlling the relative magnitude of the target images’ gradient signal.

To understand this construction, we consider a simple example case of two images $x_{i_1}, x_{i_2}$ where the average of feature of the two images at entry $j$ is known to be $\overline{f}_j = \frac{f_{i_1, j} + f_{i_2, j}}{2}$, and for simplicity, we assume $f_{i_1, j} < f_{i_2, j}$. A malicious server can choose $\theta$ equal to $\overline{f}_j$ so that with this feature fishing strategy, $\langle W_c, f_{i_1} \rangle + b_c = -l < 0$ and $\langle W_c, f_{i_2} \rangle + b_c = l > 0$, where $l = (f_{i_2, j} - \overline{f}_j) * \beta$. In this case, if the malicious server chooses a large enough $\beta$, the ratio of the ground-truth class' softmax entry for $x_{i_1}$ coming from the target class is much lower than the ground-truth class' softmax entry for $x_{i_2}$. 
Such cutoff also can be applied to more than two images in a straightforward way. We theoretically motivate this strategy in \cref{prop2} by showing that, if implemented properly, the server can often diminish the impact of all but one image on the gradient contribution for the target class. 

\begin{proposition}
\label{prop2}
If the server knows the CDF (assumed to be continuous) for the distribution of any feature of interest for the target class, then for any batch of data $\{x_i\}_{i=1}^N$ with $\{x_{i_j}\}_{j=1}^{M \leq N}$ images coming from the target class, if the gradient of $\mathcal{L}$ w.r.t. each data point in the batch is bounded in $l_2$ norm (where $\mathcal{L}$ is cross entropy loss), then we have that for the proposed feature fishing strategy, with probability $p\geq \frac{1}{e}$,  there exists a data point $x_* \in \{x_{i_j}\}$ so that

\[
\lVert g_c - \frac{1}{N}\nabla_\phi \mathcal{L}(x_*, c; \phi, W, b) \rVert_2 \leq \varepsilon,
\]

where $g_c$ is the component of the model update coming from data from the target class.
\end{proposition}

The missing ingredient for the server now is \textit{how} to find the distribution of a feature of interest from the users, so that an optimal cutoff can be chosen as visualized in \cref{fig:feature_teaser}. What helps the server here is that the mean feature value over a batch of data can always be recovered from the gradient of such data analytically \citep{phong_privacy-preserving_2017-1,geiping_inverting_2020-1}. This allows the server to measure the feature live, without the need for any auxiliary data. In the following sections we describe schemes to estimate the optimal cutoff in cross-device and cross-silo scenarios.

\begin{remark}
In order to successfully implement this feature fishing strategy, the attacker needs to be able to magnify the logits of a chosen feature and class over other classes. For example, to breach privacy for a fish image (as in \cref{fig:feature_teaser}), the attacker must magnify the logits corresponding to the ``fish" class for all images except the one that is ultimately compromised. In the scenario when there are other classes present in a user update, the attacker must accomplish this without inadvertently increasing the loss for non-targeted classes. A sufficient condition for an attacker to do this is to find a feature along which some cutoff can be made so that all data from a single class alone lies on one side of this cutoff. This allows the attacker to freely manipulate the logit for this class without increasing the loss for non-targeted data. Empirically, we find that this happens naturally in later rounds of training, especially when choosing to attack a feature with a large average value. 
However, we believe that future adjustments to our strategy are possible so that the feature fishing can occur at any stage of training. 
\end{remark}

\subsection{Cross-device}
\label{cross_device}
In this subsection, we will illustrate how the strategies of class fishing and feature fishing enable a malicious server to successfully breach privacy in cross-device federated learning \citep{kairouz_advances_2021}, where a server has an opportunity to train a model on the data from a massive number of users, yet the server might never visit the same user twice.



Because of the large number of users in cross-device FL, we propose an attack that first estimates the feature distribution from some group of users, and then applies the feature fishing strategy to catch only one data point from each of the remaining users. We call this method one-shot feature attack.

With the class fishing strategy, a malicious server can easily get the average feature of the images with the target class $c$ from a batch: Given the gradient updates for weight and bias in the classification layer, $\nabla_W \mathcal{L}$ and $\nabla_b \mathcal{L}$, the average feature can be recovered as \citep{phong_privacy-preserving_2017-1}:
\[
\overline{f} = \frac{\nabla_{W_c} \mathcal{L}}{\nabla_{b_c} \mathcal{L}}.
\]
A malicious server first applies the class attack to recover collections of average features $\{\overline{f}_{i}, s_i\}^{N}_{i=1}$, where $s_i$ is the number of images with target class $c$ for user $i$. Out of the $m$-dimensional feature vectors recorded in this manner, the attacker needs to choose a single feature for the attack. We argue that a strong choice for the attacker is to choose this feature $j$ via:
\[
j = \argmin_k \ks({\norm(\overline{f}_{i, k}),\,\mathcal{N}(0,\,1)}),
\]
where $\norm()$ evaluates the CDF of normalized data and $\ks()$ returns the statistic of the Kolmogorov-Smirnov test on two CDFs. Then, the CDF of feature entry $j$ is estimated as a normal distribution $\mathcal{N}(\mu,\,\sigma^2)$ with $\mu = \frac{1}{N}\sum_{i=1}^N\overline{f}_{i, j}$ and $\sigma = \sqrt{\frac{\sum_{i=1}^N(\overline{f}_{i, j}  \sqrt{s_i} - \mu')^2}{N}}$, where $\mu' = \frac{1}{N}\sum_{i=1}^N(\overline{f}_{i, j}  \sqrt{s_i})$.
%

The attack employs the Kolmogorov-Smirnov test to quantify the ``normality" of the distribution of the chosen feature. Overall, the attacker has limited capacity to accurately model a complicated feature distribution, which might occur for a variety of reasons. For example, for many classification models like ResNets \citep{he_deep_2015}, the ReLU before the classification can lead to a bi-modal distribution for some features, with a large mode at $0$. However, the attacker can circumvent this by simply choosing the most normal feature via the KS-Test.
For example, in \cref{fig:one_shot_est_kstest}, the CDF of the feature with KS-statistic $0.298$ is hard to estimate using the sampling means, but the feature with KS-statistic $0.031$ is almost perfectly normally distributed and can be estimated effectively from limited measurements $\overline{f}_{i,k}$.
\begin{figure}
    \centering
    \includegraphics[width=0.5\textwidth]{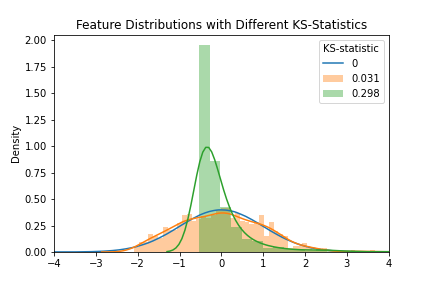}
    \vspace{-0.25cm}
    \caption{The comparison between normalized feature distributions with different KS-statistics. In the feature space of classification models, features exist which are abnormally distributed, i.e. the feature with KS-statistics of $0.298$, but other features are close to normal, i.e. with KS-statistics of $0.031$, and can be well estimated by an attacker.}
    \vspace{-0.2cm}
    \label{fig:one_shot_est_kstest}
\end{figure}

\subsection{Cross-silo}
\label{cross_silo}
In cross-silo federated learning \citep{kairouz_advances_2021}, a server is generally querying all or most users in each round. Each user is identifiable by the server. In this setting, there might not exist enough users to reliably estimate the feature distribution from their updates, and feature distributions might be significantly different between participants. However, in this scenario, an attacker can abuse the ability of the server to query in every round to iteratively update their attack. Such a \textit{binary attack} can eventually retrieve all gradients of each individual data point on the user side. The overall idea of this strategy is that if a malicious server moves the decision boundary to the average feature value $\overline{f}_j$, then they can wait for the next user update and retrieve the average feature of data points with feature value less than $\overline{f}_j$ and the average feature of data points with feature value more than $\overline{f}_j$, as well as their averaged gradient update. Iteratively, the server is thus able to construct a binary tree to find all cutoffs, retrieve all cumulative averaged gradients, and calculate the single gradient for each data point. Detailed implementation can be found in \cref{alg:binary_attack}. Instead of a full binary attack over a large number queries, the server can also aim to recover one data point, which we will call the one-shot binary attack and visualize in \cref{fig:feature_teaser}.


Under the assumption that attacked user holds $n$ data points with the target class $c$, and the feature is uniformly distributed at selected feature entry $j$, the binary attack requires $\mathcal{O}(n\log n)$ queries on average to retrieve all separate per-example gradients, and the one-shot binary attack requires $\mathcal{O}(\log n)$ queries on average. In the worst case where each cut of the binary tree only removes a single data point,  $\mathcal{O}(n^2)$ queries are necessary for the full binary attack and $\mathcal{O}(n)$ for the one-shot binary attack. However, we find this scenario empirically less likely as, often, the feature distribution appears normal.

\begin{algorithm}[tb]
   \caption{Binary Attack}
   \label{alg:binary_attack}
\begin{algorithmic}[1]
\STATE \textbf{Input:} model $m$, $\theta$, target class $c$, feature index $j$, batch size $n$
\STATE $\text{visited} \gets$ Empty Set
\STATE $G \gets$ Empty OrderedDict in descending order
\STATE $C \gets$ $[+\infty]$
\WHILE{$C$ is non-empty}
    \STATE $C_\text{new} = [\,]$
    \FOR{$c_\text{cut}$ in $C$}
        \STATE $m$ = $\text{FeatureFishing}(m, \beta = c_\text{cut}, \theta = \theta)$
        \STATE $g = $ gradient update from the user with model $m$
        \STATE $f = \nabla_{W_{c, j}} \mathcal{L} / \nabla_{b_c} \mathcal{L}$
        \IF{$f$ not in $\text{visited}$}   
            \STATE $\text{visited}$ add node $f$
            \STATE $G[f] = g \times n$
            \STATE $C_\text{new}$ append $f$
            \STATE $C_\text{new}$  append $2 \times c_\text{cut} - f$
        \ENDIF
        \STATE $C = C_\text{new}$
    \ENDFOR
\ENDWHILE
\STATE $G$ $\gets$ values of $G$
\STATE $g_\text{single} = [G[0]]$
\FOR{$i$ in $1...\text{len}(G)$}
    \STATE $g_\text{single}$ append $G[i] - G[i - 1]$
\ENDFOR
\STATE \textbf{return} $g_\text{single}$
\end{algorithmic}
\end{algorithm}

\section{Experiments}
\label{experiments}
In this section, we evaluate the empirical threat of these attacks. We show that fishing attacks can breach the privacy even for user aggregates with large batch sizes. We also provide quantitative comparisons with prior work.

\subsection{Experimentation details}
\label{exp_detail}
For our experiments, we focus on image classification as the central task. All images are from ImageNet ILSVRC 2012 \citep{russakovsky_imagenet_2015} with a size of $224 \times 224$ and include $1000$ classes in total. We use $\alpha = 1000$ for the class fishing strategy and $\theta = 1000$ for the feature fishing strategy, and these hyperparameters are large enough to suppress the $\varepsilon$ bound on the ``unwanted" gradients to be negligibly small in practice. We apply both strategies to the last linear layer of a pre-trained ResNet-18 for all experiments except \cref{april}. We also use the optimization method of \citet{geiping_inverting_2020-1} to recover the image from the gradient update caught by the attack. In the optimization, we use Adam with step size $0.1$ and $50$ iterations of warmup over total $24$K \citep{yin_see_2021}. The initialization is set to the pattern tiling of $4\times4$ random normal data introduced in \citep{wei_framework_2020}. For regularization, we use a double-opponent, isotropic vectorial total variation \citep{astrom_double-opponent_2016}.


\subsection{One-shot feature attack}
\label{one_shot_feature_exp}
For the one-shot feature attack employed in the cross-device setting, we focus on investigating how the number of users a malicious server estimates influences the probability of retrieving a separate data point. We test our estimation at several levels of number of users ranging from $50$ to $225$ (this range is chosen for simplicity, as discussed in \cref{sec:feature_strategy} it not an upper bound on the number of users). For each user, we consider $10$ images in total, and $4$ images from the target class. After the feature estimation, we then run the attack on $100$ other, separate users also with batch size $10$, with images all sampled from unseen images within the target class. For our experiments, we declare that the attacker recovers $n$ data points $\{x_i\}_{i=1}^n$ if
\[
\similarity(g_c, \sum_{i=1}^n\nabla_\phi \mathcal{L}(x_i, c; \phi, W, b)) \geq 0.95,
\]
where $g_c$ is the gradient update and $\similarity$ is the cosine similarity. Predictably, we find in \cref{fig:bs_vs_dist} that a server better estimates the feature distribution of interest with a larger number of users. As stated in \cref{prop2}, if the server accurately estimates the distribution of the chosen feature, they can expect to recover a user image $\frac{1}{e} \approx 37\%$ of the time! 

\begin{figure*}[h]
    \centering
    \subfigure[]{\includegraphics[width=0.49\textwidth]{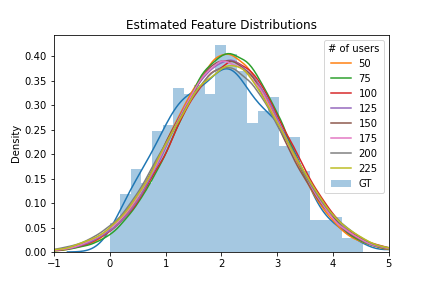}\label{fig:bs_vs_dist}}
    \subfigure[]{\includegraphics[width=0.49\textwidth]{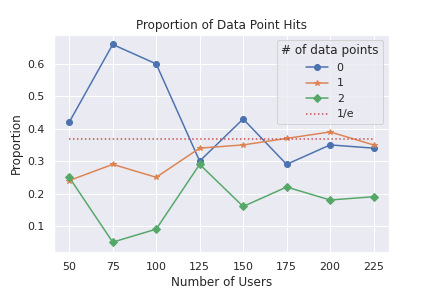}\label{fig:bs_vs_prob}}
    \caption{Influence of the number of users on the feature estimation. (a) The comparison of ground truth feature distribution and Gaussian fit to estimated feature distribution with different numbers of users. The more number of users a malicious server estimates the more estimated distribution closer to the ground true distribution. (b) The comparison of the proportions of getting $n$ data points across various numbers of users. Here the server wants exactly $1$ datapoint to fall across the selected cuts.}
    \label{fig:feat_est_bs}
\end{figure*}

Further, we include an ablation study regarding the importance of the KS-test in Tab.~\ref{table:ks-test}. We test out other heuristic methods to choose the feature entry, like the feature entry with highest standard deviation. Choosing the most normal feature is a helpful heuristic to improve attack success. However, the KS-test is not crucial for the attack to succeed, even attacking a random feature still recovers private data.

\begin{table}[]
\caption{Probability to recover $n$ data points for various method of choosing the attacked feature entry. Recovery at $n=1$ is an immediate breach in privacy.}
\label{table:ks-test}
\centering
\vspace{3mm}
\begin{center}
\begin{tabular}{ccccc}
\hline
\multirow{2}{*}{Method} & \multicolumn{4}{c}{\# of Data Points Caught}  \\ \cline{2-5} 
                        & 0    & 1             & 2    & $\geq$3 \\ \hline
\textbf{KS-Test (Ours)} & 0.35 & \textbf{0.39} & 0.18 & 0.08             \\ \hline
Highest KS-Stat         & 1.0  & .0            & .0   & .0               \\
Lowest Mean             & 1.0  & .0            & .0   & .0               \\
Highest Mean            & 0.26 & 0.37          & 0.25 & 0.12             \\
Lowest STD              & 1.0  & .0            & .0   & .0               \\
Highest STD             & 0.39 & 0.31          & 0.22 & 0.08             \\
Random                  & 0.7  & 0.19          & 0.09 & 0.02             \\ \hline
\end{tabular}
\end{center}
\vspace{-0.75cm}
\end{table}

\subsection{Binary attack}
\label{binary_attack_exp}
We also verify the stability of the one-shot binary attack employed in the cross-silo setting experimentally and evaluate how many queries a malicious server needs to breach the privacy of a user. For different batch sizes, we sample $50$ different users who each own only data points with the same label (so that the maximal number of label collisions occurs) and observe the number of queries for the one-shot binary attack. We do indeed observe  in \cref{fig:one_shot_binary_query} that the number of necessary queries scales as $\log(n)$. For example, this allows an attacker to retrieve an individual data point from an aggregation of $256$ images after, on average, $8$ queries. Given that applications, e.g. in \citet{mcmahan_learning_2018}, run over 100 rounds of FL training, this is well within the range of possibility (Over $100$ rounds, the potential aggregation size that could be breached is $2^{100}\sim1e30$ data points).

In all cases, the underlying attack we employ is the optimization recovery method of  \citet{geiping_inverting_2020-1}. We do this for the update retrieved by the one-shot binary attack with different batch sizes on groups of users generated for ImageNet data by random partitions. We plot Max-RPSNR and Max-SSIM in \cref{fig:one_shot_binary}, which represent the maximum registered PSNR and maximum SSIM value in the batch for each batch during the experiment. \cref{table:compare_prev} also provides the results between the one-shot binary attack and original attack in \citet{geiping_inverting_2020-1}. With the one-shot binary attack, we can see a noticeable improvement on the original attack.
We note that with increasing batch size, this improvement actually increases further as the larger batches contain more vulnerable data points that are empirically easier to reconstruct. 

\begin{figure*}[h]
    \centering
    \subfigure[]{\includegraphics[width=0.49\textwidth]{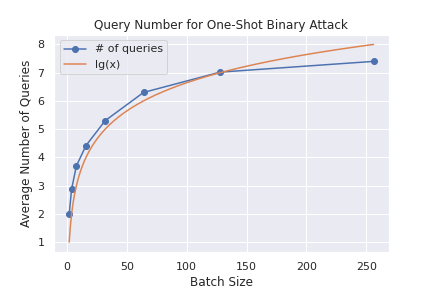}\label{fig:one_shot_binary_query}}
    \subfigure[]{\includegraphics[width=0.49\textwidth]{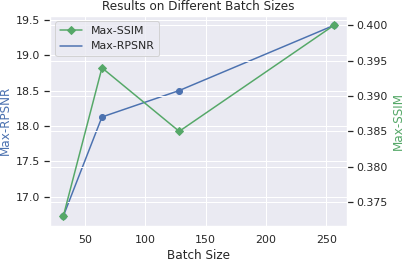}\label{fig:one_shot_binary}}
    \caption{The performance of the one-shot binary attack. (a) The average number of queries required to breach one data point in various batch sizes for the one-shot binary attack. Each data point is the average over $50$ users with separate labels. (b) Quantitative results leveraging the attack of \citet{geiping_inverting_2020-1} on large ImageNet batches.}
    \label{fig:one_shot_binary_fig}
\end{figure*}

\begin{table}[]
\centering
\caption{The numerical results between previous work \citep{geiping_inverting_2020-1} and the one-shot binary attack.}
\vspace{3mm}
\label{table:compare_prev}
\begin{tabular}{cccc}
\hline
\multirow{2}{*}{Metric}    & \multirow{2}{*}{Method} & \multicolumn{2}{c}{\textbf{Batch Size}} \\ \cline{3-4} 
                           &                         & 32                 & 50                 \\ \hline
\multirow{2}{*}{Max-RPSNR} & Prev.                    & $14.029$             & $14.568$             \\
                           & \textbf{Ours}                    & \textbf{17.415}             & \textbf{17.610}             \\ \hline
\end{tabular}
\end{table}

\subsection{Fishing with closed-form APRIL attacks}
\label{april}
Because our proposed method is attack-agnostic (it can be combined with prior optimization or analytic attacks), we also reproduce the experiments from a recent closed-form attack, APRIL \citep{lu_april_2021}, which targets vision transformer models \citep{dosovitskiy2020image}. In APRIL, a malicious server is able to solve the closed-form solution of the self-attention layer to recover the patch embeddings, and then recover the original image by inverting the patch embedding layer. This method works exceedingly well for a single data point, but fails when the batch size is greater than $1$, as then, patches from multiple images are mixed. We first modify a ViT-Small \citep{dosovitskiy2020image} to build the setting described in \citet{lu_april_2021} by removing the Layer norm and the residual connection from the first block. Then, we run both attacks on $50$ users for each batch size. \cref{table:april_result} shows the results across different batch sizes on ImageNet. The original APRIL attack fails when the batch size reaches $2$, further and further mixing patches until the recovered image is totally irrecoverable. However, the proposed fishing attack remains consistent with larger batch size, allowing this attack to scale to arbitrary batch sizes.

\begin{table}[]
\caption{Quantitative (Max-RPSNR) and qualitative comparison between original APRIL and APRIL + class fishing strategy on modified ViT-Small. The original APRIL attack fails when the batch size is barely $2$, but fishing method with APRIL can still breach privacy at a batch size of $256$.}
\centering
\vspace{3mm}
\label{table:april_result}
\begin{tabular}{ccccccc}
\hline
\multirow{2}{*}{\textbf{Method}}          & \multicolumn{4}{c}{\textbf{Batch Size / Max-RPSNR}} \\ \cline{2-5} 
                                 & 1    & 2    & 64  & 256  \\ \hline
\multicolumn{1}{c}{APRIL}        & \textbf{23.257}    & $12.664$    & $8.318$   & $8.422$    \\
\multicolumn{1}{c}{\textbf{Ours}} & $23.154$    & \textbf{21.047}    & \textbf{21.187}   & \textbf{21.481}    \\ \hline
APRIL                                                                    &
\includegraphics[align=c, scale=0.16]{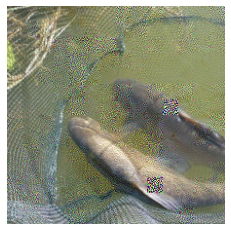}  & 
\includegraphics[align=c, scale=0.16]{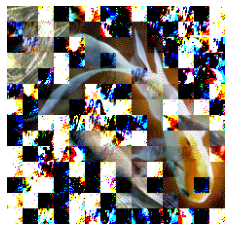}  & 
\includegraphics[align=c, scale=0.16]{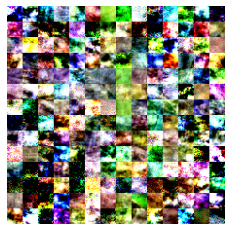} & 
\includegraphics[align=c, scale=0.16]{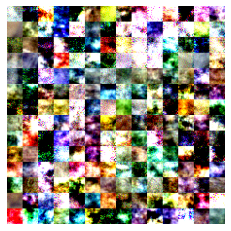}\\
\textbf{Fishing}                                                            & 
\includegraphics[align=c, scale=0.16]{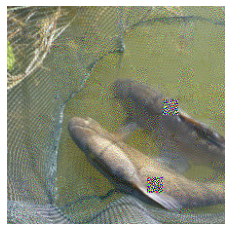}    & 
\includegraphics[align=c, scale=0.16]{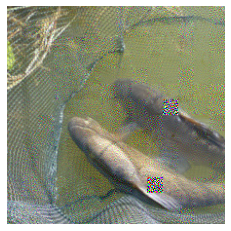}    & 
\includegraphics[align=c, scale=0.16]{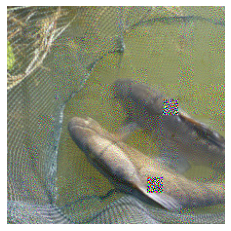}   & 
\includegraphics[align=c, scale=0.16]{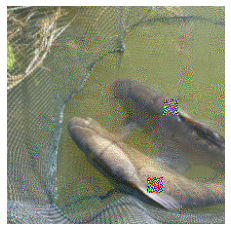}
\end{tabular}
\end{table}

\section{Defenses and Mitigation Strategies}
Under the proposed threat model, users have few defenses against this attack, as further aggregation does not guarantee protection. However, other forms of protection, or at least identification, may be possible. The attack investigated in this work, and the attacks of \citet{boenisch_when_2021} and \citet{pasquini_eluding_2021} work by producing sparse gradients. This means that the update returned by the user contains parameters where their gradients are dominated (or entirely contributed) by single examples. In theory, a user could attempt to measure and detect this attack pattern (inspecting and storing all per-example gradients) and refuse a model update. However, this would also flag natural examples, as this pattern also appears in benign models \citep{boenisch_when_2021}. At first glance, a more promising direction could be to aggregate based on a per-parameter median instead of mean, yet it is still unclear if this is at all achievable in a secure aggregation framework where the median computation has to be online and facilitated without a central party. Another potential limitation of median aggregation is that there could be parameter modifications in which the targeted data point is returned as the median, breaching privacy again.

Overall, the best defense for a user of federated learning is local differential privacy \citep{bhowmick_protection_2019}, which computes per-example gradient clipping and noise directly on the user side. A sufficient level of differential privacy breaks attacks against privacy, such as this one. However, model accuracy tend to diminish as privacy guarantees are increased \citep{jayaraman_evaluating_2019,bagdasaryan_differential_2019}.

\section{Disparate Impact}
We further want to highlight another concern that is unveiled by this attack. For the default threat model, privacy is non-uniform. The fishing attack generally targets outliers in feature space first. These outliers can correspond to underrepresented groups or different data domains in the dataset. For example, in the cross-silo setting, it will take on average $n\log(n)$ queries to recover the data point whose feature is at the median of all features in the update, whereas the outlier feature with the smallest value is recovered in only $\log(n)$ queries on average, as we also verify in \cref{fig:one_shot_binary_query}. Likewise the attack in the cross-device setting most easily fishes off the data point with the lowest feature value for each user/aggregation of users it attacks. To the best of our knowledge, this is the first time that disparate impact in privacy attacks has come up as a concern and want to highlight this for future investigations.

\section{Conclusions}
We investigate the threat model of malicious, or untrusted, server updates sent to user in federated learning, and find that this threat model allows attacker to leverage existing attacks to fish for separate data points in arbitrarily large batches. This requires no modification of the model, only of the feature of the last layer, and can break user privacy in both cross-device and cross-silo scenarios. A defense utilizing aggregation alone is not sufficient to protect user privacy against these attacks.

\section{Acknowledgements}
This work was support by DARPA’s GARD program (HR00112020007), the AFOSR MURI program, and the National Science Foundation (DMS-1912866).
\bibliography{icml_manual, zotero_library}
\bibliographystyle{icml2022}

\newpage
\appendix
\onecolumn

\section{Proofs.}
\label{proofs}

\begin{proposition*}{\textbf{(\cref{prop1})}}
Let $\phi$ denote the parameters of the feature extractor for the central model, and let $W, b$ be the weight and bias vectors, respectively, of the last classification layer. Furthermore, let $\{x_i\}_{i=1}^N$ be the collection of user data on which model updates can be calculated. Let $\{x_{i}\}_{i=1}^{m \leq N}$ be any batch data containing $k$ data points $\{x'_i\}^{k\leq m}_{i=1}$ from the target class. Then, for cross-entropy loss $\mathcal{L}$, if the gradient of $\mathcal{L}$ w.r.t. each data point in the batch is bounded in $l_2$ norm, then $\forall \varepsilon > 0$, $\exists \alpha$ so that a malicious server employing our class strategy on $W, b$ recovers a feature extractor update $g$ so that 
\[
\lVert g - \frac{1}{m}\sum_{i=1}^{k}\nabla_\phi \mathcal{L}(x'_i, c; \phi, W,b)\rVert_2 < \varepsilon.
\]

\end{proposition*}

\begin{proof}
Let $\{f_i\}_{i=1}^m = \{ \phi(x_i)\}_{i=1}^m$ be the features of the batch of data used for update $g$. Similarly, let $\{p_i\}_{i=1}^m$ be the collection of softmax outputs for these features (i.e. $p_j = \sigma(Wf_j + b) = \sigma(l_j)$ for $\sigma$ the softmax function, $l_j$ the corresponding logits). We want to show that the gradient contribution for the images coming from the non-target classes is minimal. To see this, WLOG, assume the first $k$ images from the batch are images from the targeted class. Then, we separate the model update as:

\[ 
g = \frac{1}{m} \bigg( \sum_{i=1}^{k} \nabla_\phi \mathcal{L}(x_i, c; \phi, W, b) + \underbrace{\sum_{j=k+1}^m \nabla_\phi \mathcal{L}(x_j, y_j; \phi, W, b)}_{(*)} \bigg) 
\]

We now equivalently have to show that $\frac{1}{m}\lVert (*) \rVert_2 \leq \varepsilon$. To do this, we take 
$x_* = \argmax_{x \in \{x_j\}_{j=k+1}^m} \lVert \nabla_\phi \mathcal{L}(x_j, y_j; \phi, W, b)\rVert_2$. Now, it suffices to show that $\lVert \nabla_\phi \mathcal{L}(x_*, y_*; \phi, W, b)\rVert_2 < \varepsilon$. For cross-entropy loss, the only softmax prediction used in the loss corresponds to the ground-truth label. So for logits $\{l_i \}$, so label space of size $s$, we have:  
\[
    L(x_*, y_*;\phi, W, b) = -\log(p_{*,y_*}) = -\log \bigg( \frac{e^{l_{*,y_*}}}{\sum_{i=1}^s e^{l_{*,i}}} \bigg)
\]

for any index $s \neq c$ for the logits $l_*$, by construction of the matrix $W$, we have $\nabla_{\phi(x_*)} l_{*,s} = W_s = \Vec{0}$. Therefore, when determining contributions to the gradient vector for $\phi$, we need only account for gradient contributions from the target class' logit. From here, a straightforward calculation shows that.  

\[
\frac{\partial \mathcal{L}}{\partial l_{*,c}} = \frac{\partial \mathcal{L}}{\partial p_{*,y_*}} \frac{\partial p_{*,y_*}}{\partial l_{*,c}} = \frac{-1}{p_{*,y_*}} \cdot \frac{\partial p_{*,y_*}}{\partial l_{*,c}} = \underbrace{\frac{-1}{p_{*,y_*}} \cdot \frac{\partial}{\partial l_{*,c}} \bigg( \frac{e^{l_{*,y_*}}}{\sum_{i=1}^s e^{l_{*,i}}} \bigg)}_{(\#)}
\]

Now there are two cases for $(\#)$: 

Case 1: 
$y_* \neq c$: 
\[
\frac{-1}{p_{*,y_*}} \cdot \frac{\partial}{\partial l_{*,c}} \bigg( \frac{e^{l_{*,y_*}}}{\sum_{i=1}^s e^{l_{*,i}}} \bigg) = \frac{-1}{p_{*,y_*}} \cdot  \frac{-e^{l_{*,y_*}} \cdot e^{l_{*,c}}}{\big(\sum_{i=1}^s e^{l_{*,i}} \big)^2} = \frac{1}{p_{*,y_*}} \cdot p_{*,y_*} \cdot p_{*,c} = p_{*,c}
\]

Case 2: 
$y_* = c$: 
\[
\frac{-1}{p_{*,c}} \cdot \frac{\partial}{\partial l_{*,c}} \bigg( \frac{e^{l_{*,y_*}}}{\sum_{i=1}^s e^{l_{*,i}}} \bigg) = \frac{-1}{p_{*,c}} \cdot  \frac{e^{l_{*,c}} \cdot \big(\sum_{i=1}^s e^{l_{*,i}}\big)-e^{2\cdot l_{*,c}}}{\big(\sum_{i=1}^s e^{l_{*,i}} \big)^2} = \frac{-1}{p_{*,c}} \cdot p_{*,c} \cdot (1-p_{*,c}) = p_{*,c} - 1
\]

We focus on Case 1 since this determines the gradient magnitude for $(*)$. For this case, we notice that:

$$p_{*,c} = \frac{e^{l_{*,c}}}{\sum_{i=1}^s e^{l_{*,s}}} \leq \frac{e^{l_{*,c}}}{(s-1) \cdot e^\alpha}$$

so if the server sets $\alpha = \gamma \cdot \max_i \max_s {f_{i,s}}$ - i.e. the max logit over the entire batch - for some scale factor $\gamma$, this simplifies to $p_{*,c} \leq \frac{1}{(s-1)\cdot e^{(\gamma-1)l_{*,c}}}$. Putting this all together, we have: 

\[
\lVert \nabla_\phi \mathcal{L}(x_*, y_*; \phi, W, b)\rVert_2 = \lVert\frac{\partial \mathcal{L}}{\partial l_{*,c}} \cdot \nabla_\phi l_{*,c}\rVert_2 \leq \lVert \frac{1}{(s-1)\cdot e^{(\gamma-1)l_{*,c}}} \cdot \nabla_\phi l_{*,c}\rVert_2
\]

By straightforward manipulation from this point, it is clear that the server can scale $\gamma$, and therefore $\alpha$ appropriately to ensure that 
\[ 
\lVert \nabla_\phi \mathcal{L}(x_*, y_*; \phi, W, b)\rVert_2 \leq \epsilon
\] which completes the proof. 

Note that the derivation for Case 2 is not necessary for the statement of the proposition, however it does reveal that the gradients for images coming from the target class are not suppressed in magnitude in the same way that the non-target class gradients are, and, in fact, we see that for any image $x_i$ from the target class, and any image $x_j$ not from the target class:

\[
\lim_{\alpha \to \infty} \bigg | \frac{\frac{\partial \mathcal{L}}{\partial l_{i,c}}}{\frac{\partial \mathcal{L}}{\partial l_{j,c}}} \bigg | = \lim_{\alpha \to \infty} \bigg | \frac{p_{i,c} - 1}{p_{j,c}} \bigg | = \infty
\]
\end{proof}

\begin{proposition*}{\textbf{(\cref{prop2})}}
If the server knows the CDF (assumed to be continuous) for the distribution of any feature of interest for the target class, then for any batch of data $\{x_i\}_{i=1}^N$ with $\{x_{i_j}\}_{j=1}^{M \leq N}$ images coming from the target class, if the gradient of $\mathcal{L}$ w.r.t. each data point in the batch is bounded in $l_2$ norm (where $\mathcal{L}$ is the cross entropy loss), then we have that for the proposed feature fishing strategy, with probability $p\geq \frac{1}{e}$,  there exists a data point $x_* \in \{x_{i_j}\}$ so that

\[
\lVert g_c - \frac{1}{N}\nabla_\phi \mathcal{L}(x_*, c; \phi, W, b) \rVert_2 \leq \varepsilon
\]

where $g_c$ is the component of the model update coming from data from the target class.
\end{proposition*}
\begin{proof}

If the server knows the CDF for some feature of interest for the target class, then the server can create ``bins" of equal mass corresponding to mass $\frac{1}{M}$. Let $X$ be the random variable corresponding to the number of data points from the target class in our batch that fall into the chosen bin. Then, we have:

\[
p(X=1) = \binom{M}{1} \cdot \frac{1}{M} \cdot \bigg(1 - \frac{1}{M}\bigg )^{M-1}  = \bigg(1 - \frac{1}{M}\bigg)^{M-1} \geq \frac{1}{e}
\]

Now, if we have that one and only one image, $x_*$, from the target class falls into this bin, then, by construction of our feature fishing attack, for any $\varepsilon_1, \varepsilon_2 > 0$, we can choose parameter $\beta$ so that for the ``fished" image $x_*$, we have that $p_{*,c} < \varepsilon_1$ and for any image $x_k \in \{ x_{i_j} \}$ where $x_k \neq x_*$, then $p_{k,c} > 1 - \varepsilon_2$. We can see this via the calculation in \cref{sec:feature_strategy} where the logits of the fished and non-fished logits have oposite signs and depend linearly on the server parameter $\beta$. 

Then, using a similar calculation to the proof for \cref{prop1}, we have that 

\[
\lVert g_c - \frac{1}{N}\nabla_\phi \mathcal{L}(x_*, c; \phi, W, b) \rVert_2 = \lVert \frac{1}{N} \sum_{\substack{x_k \in \{x_{i_j}\}, \\ x_k \neq x_*}} \nabla_\phi \mathcal{L}(x_k, c; \phi, W, b) 
\rVert_2\leq 
\max_{\substack{x_k \in \{x_{i_j}\}, \\ x_k \neq x_*}} \lVert \nabla_\phi\mathcal{L}(x_k, c;\phi, W, b) \rVert_2
\]
because the feature fishing attack is combined with the class fishing attack, we have the following: 
\[
= \max_{\substack{x_k \in \{x_{i_j}\}, \\ x_k \neq x_*}} \lVert \frac{\partial \mathcal{L}}{\partial l_{k, c}} \cdot \nabla_\phi l_{k,c} \rVert_2 \leq |(p_{k,c} - 1) \cdot C| \leq \varepsilon
\]

for $C$ a constant. The last inequality occurs because $p_{k,c}$ can be made arbitrarily close to $1$.

Note again that we did not use the fact that $p_{*,c}$ can be made arbitrarily close to $0$, but this is important for the empirical success of the attack as it yields a high magnitude gradient for $x_*$ which improves reconstruction.  

\end{proof}

\section{Technical Details}
We implement these attacks in a \texttt{PyTorch} framework \citep{paszke_automatic_2017} which contains all material to replicate these experiments and which we include in the supplementary material. We run all the optimization-based attacks using single \texttt{2080ti} GPUs and run the analytic attacks via APRIL on CPUs, solving the embedding layer and attention inversion under-determined problems via an SVD solver (\texttt{dgelss}). For our quantitative experiments we partition the ImageNet validation set into 100 users with the given batch size, and either allocate each user a different class or assign images to users at random (without replacement). For the optimization-based attacks, we follow the experimental protocol of \citet{geiping_inverting_2020-1}, in which batch norm layers are fixed in evaluation mode and populated from a pretrained model. In this scenario, the user does not have to send updated batch norm statistics to the server. We further assume, for simplicity, that label information is contained in the update as metadata, but note that novel label recovery algorithms as discussed in \citet{wainakh_user_2021} are highly successful, even at large aggregation sizes.

We report R-PSNR and SSIM metrics in this work. For R-PSNR (registered-PSNR), as also investigated in \citet{yin_see_2021}, we register the reconstructed image to the ground truth reference data by direction optimization using \texttt{kornia} \citep{riba_kornia_2019}. For SSIM we also compute a translation-invariant SSIM score by implementinging complex-wavelet SSIM as originally proposed in \citet{wang_translation_2005}. As underlying wavelet transform we use a dual-tree complex wavelet transform over 5 scales \citep{kingsbury_complex_2001,cotter_uses_2020} with bi-orthogonal wavelet filters.

\section{Additional Results}
\begin{figure*}[h!]
  \centering
  \includegraphics[trim={0 12cm 0 0},clip,width=0.33\textwidth]{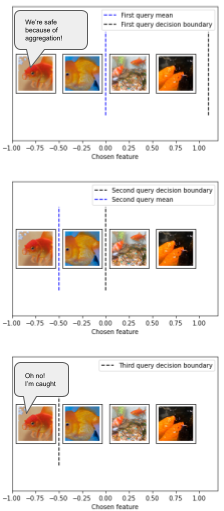}
   \includegraphics[trim={0 6cm 0 6cm},clip,width=0.33\textwidth]{assets/feature_fishing_teaser.png}
  \includegraphics[trim={0 0 0 12cm},clip,width=0.33\textwidth]{assets/feature_fishing_teaser.png}
  \caption{A high-level schematic of our feature fishing strategy. Within a class, ``Goldfish", features can be iteratively suppressed to reveal a single user image.}
  \label{fig:feature_teaser_binary}
\end{figure*}

\begin{figure*}[h]
    \centering
    \subfigure[]{\includegraphics[width=1\textwidth]{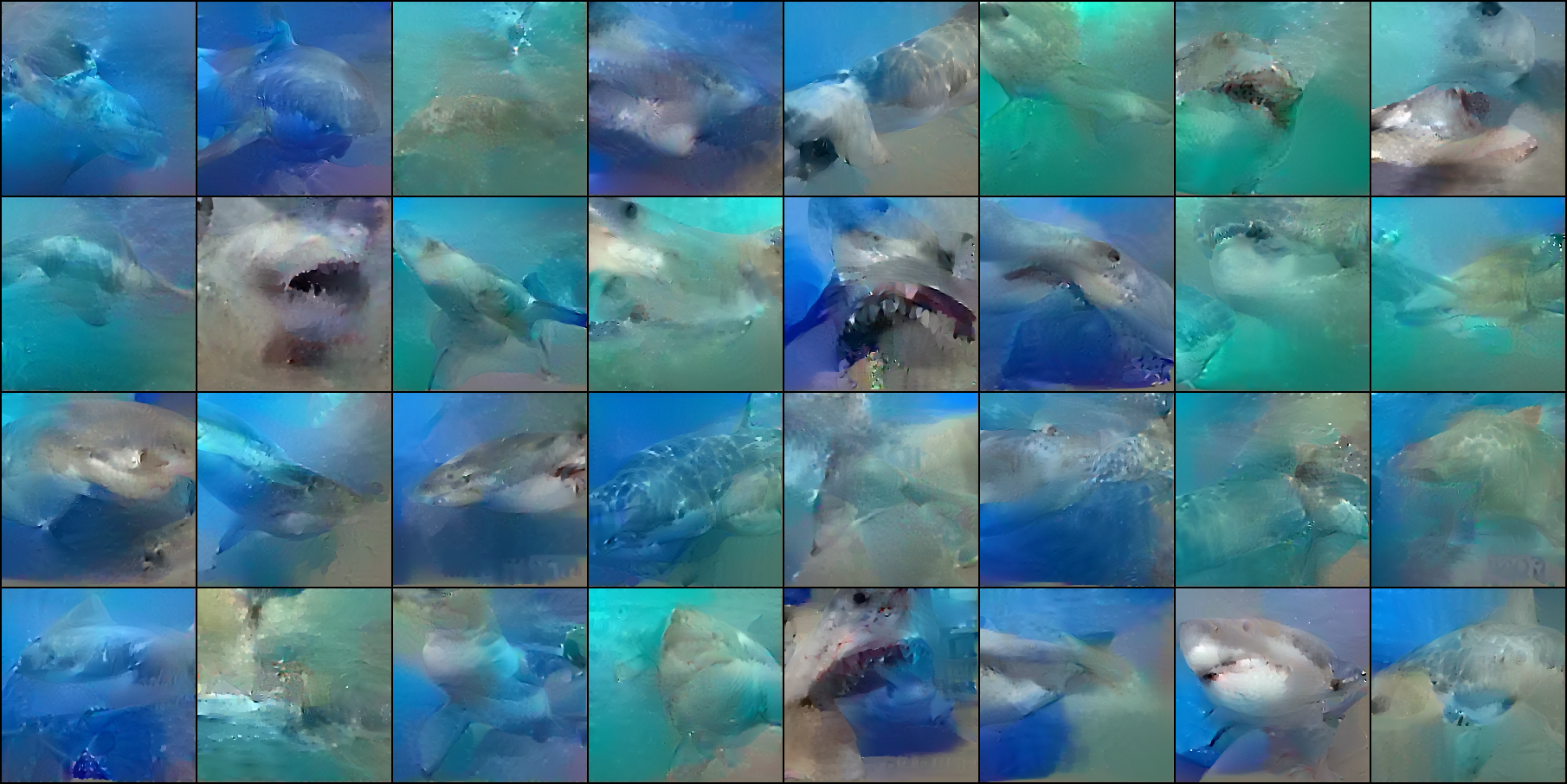}\label{fig:cls_32}}
    \subfigure[]{\includegraphics[width=1\textwidth]{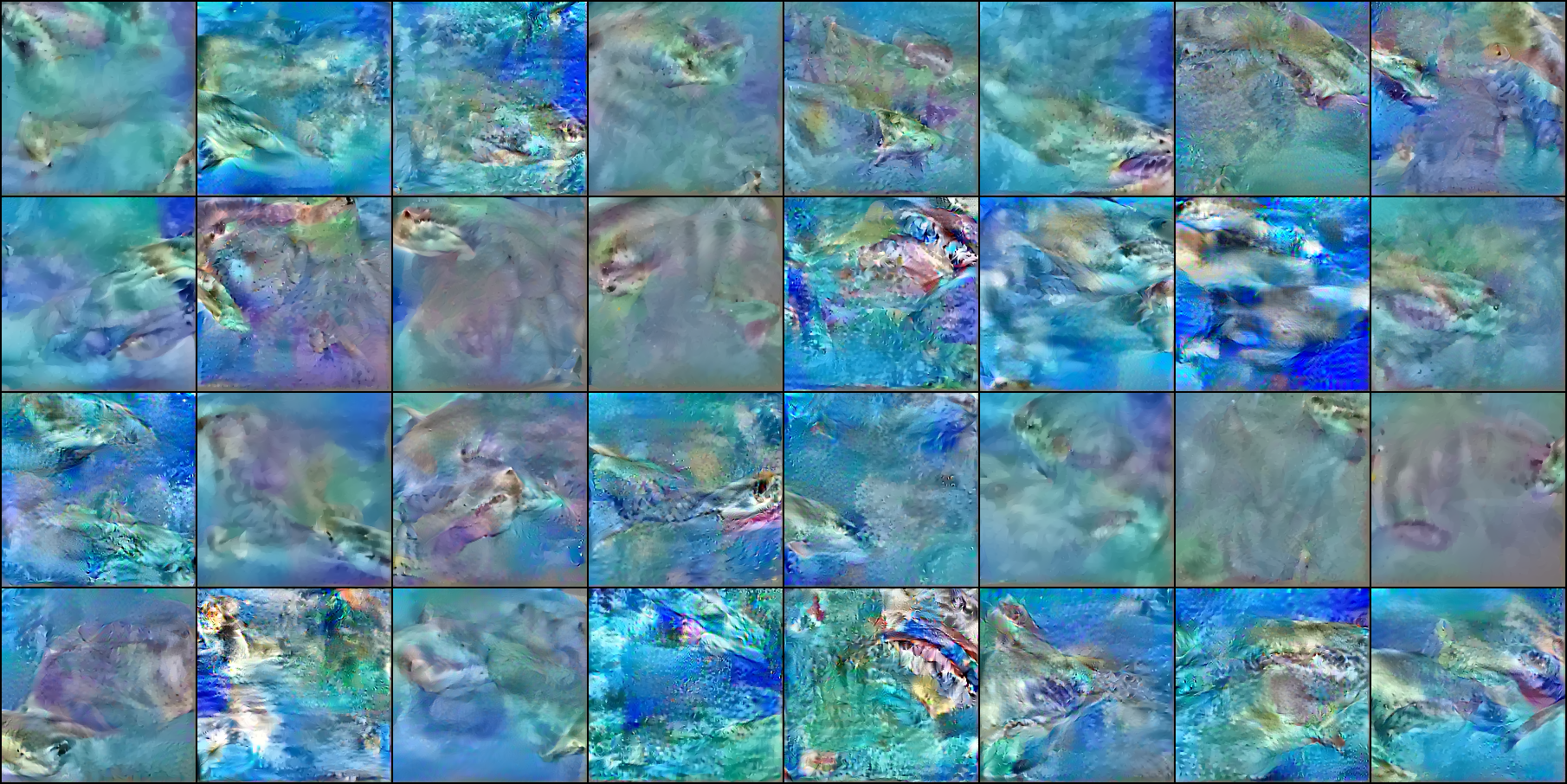}\label{fig:modern}}
    \caption{The recovered images from the batch containing class "shark" only. (a) With the class attack only without feature fishing. (b) Unmodified model.}
    \label{fig:avg_opt_result}
\end{figure*}

\begin{figure}[h]
    \centering
    \includegraphics[width=1\textwidth]{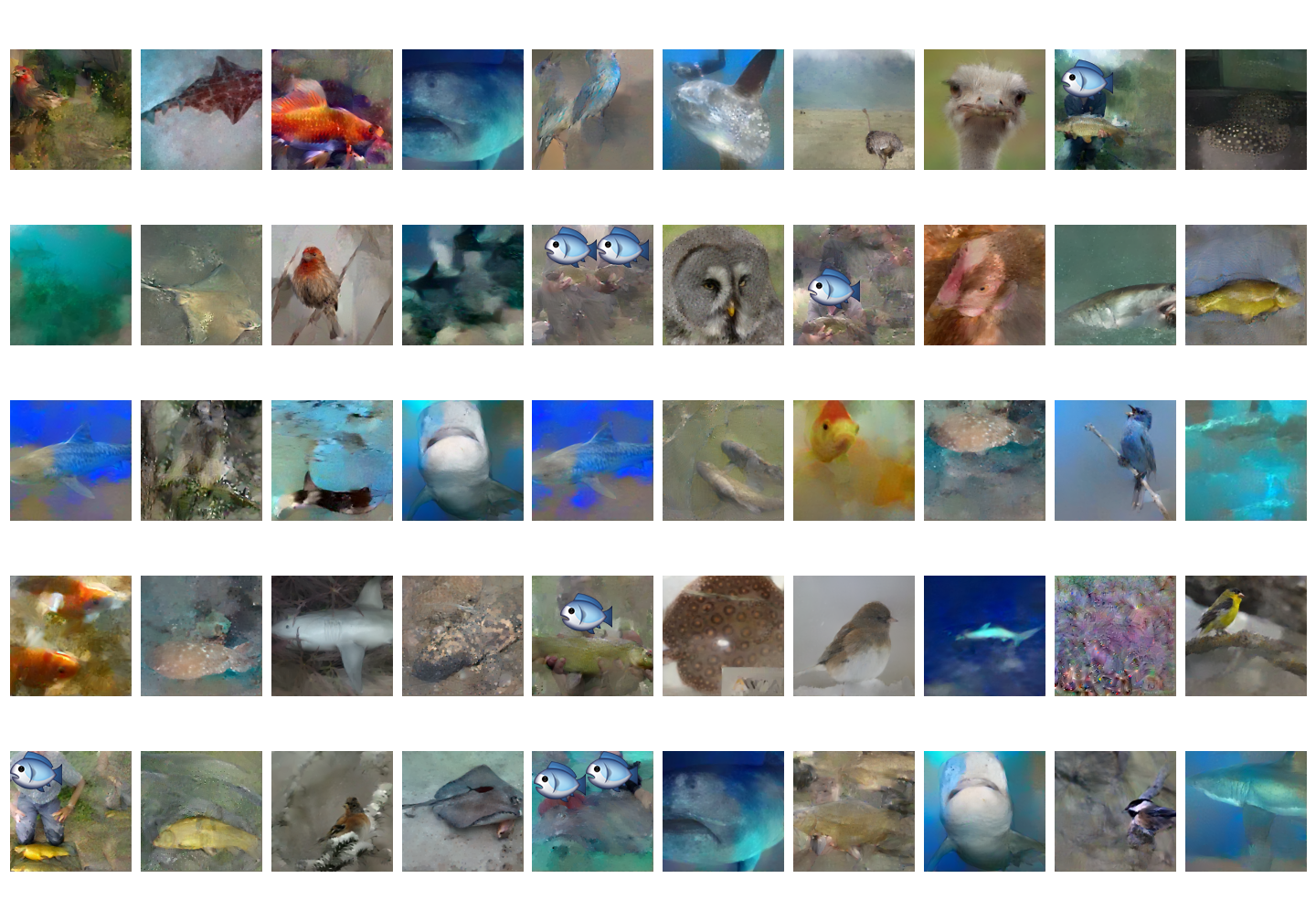}
    \caption{The recovered images from $50$ different users with batch size $128$. We add a fish emoji to human faces from the ImageNet dataset for privacy reasons. All recoveries are at least as successful as optimization-based recoveries from a single data point.}
    \label{fig:batch_128_result}
\end{figure}

\begin{figure}[h]
    \centering
    \includegraphics[width=1\textwidth]{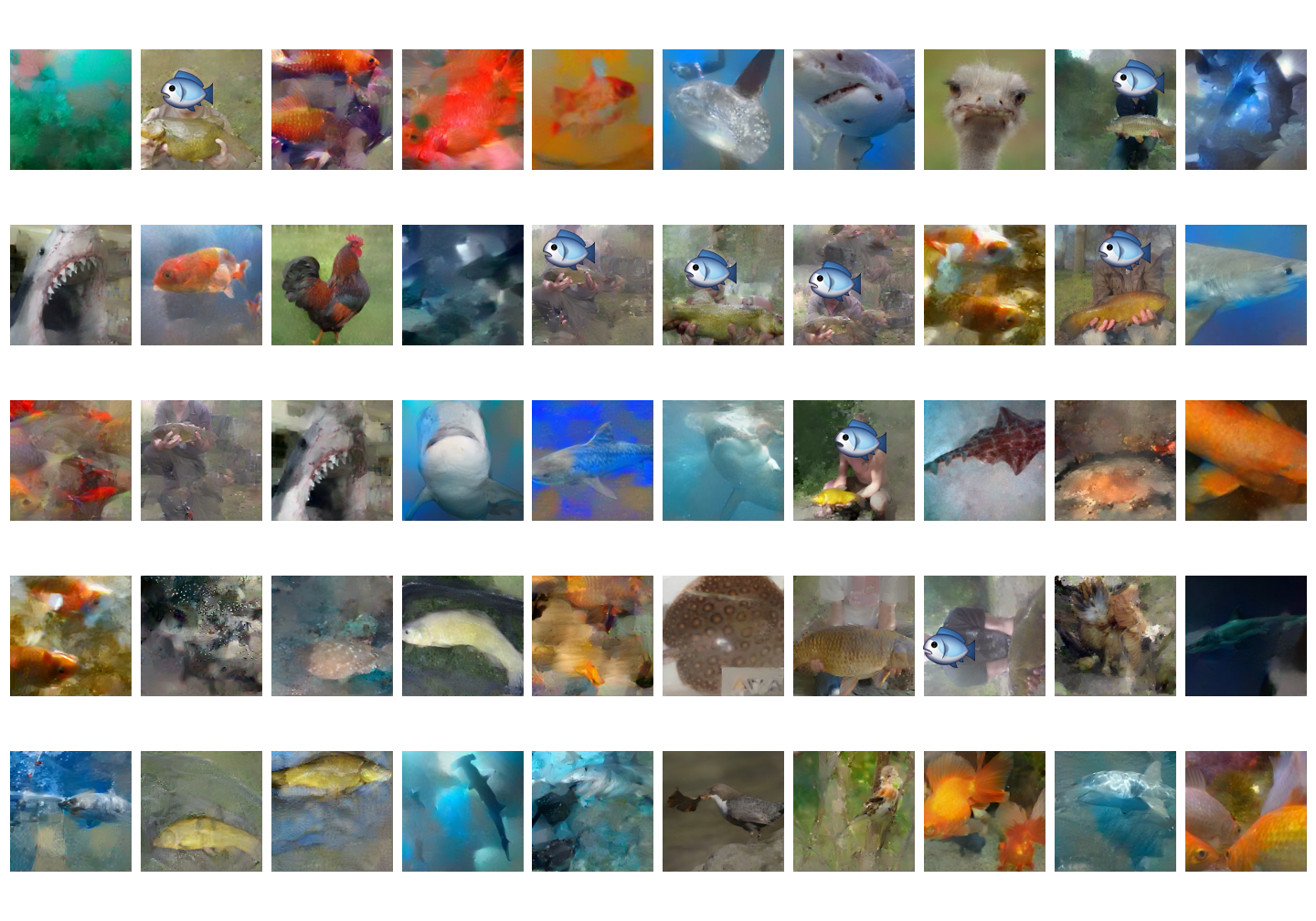}
    \caption{The recovered images from $50$ different users with batch size $256$. We add a fish emoji to human faces from the ImageNet dataset for privacy reasons. All recoveries are at least as successful as optimization-based recoveries from a single data point.}
    \label{fig:batch_256_result}
\end{figure}

\end{document}